%% file: main.tex
\documentclass{article}

\PassOptionsToPackage{numbers, compress}{natbib}
\usepackage[preprint]{neurips_2026}

\usepackage[utf8]{inputenc}
\usepackage[T1]{fontenc}
\usepackage{hyperref}
\usepackage{url}
\usepackage{booktabs}
\usepackage{amsfonts}
\usepackage{nicefrac}
\usepackage{microtype}
\usepackage{xcolor}
\usepackage{graphicx}
\input{macros}

\title{Delightful Exploration}

\author{%
  Ian Osband \\
  Google DeepMind \\
  \texttt{iosband@google.com} \\
}

\begin{document}
\maketitle

\begin{abstract}
Most exploration algorithms search broadly until uncertainty is resolved.
When the action space is too large to resolve within budget, practitioners default to $\varepsilon$-greedy, which bounds disruption but spends its override blindly.
We introduce \textit{Delight-gated exploration} (DE), a host--override rule that spends exploratory actions only when their prospective delight (expected improvement times surprisal) exceeds a gate price.
This practical heuristic recovers a classical result: Pandora's reservation-value rule for costly search, with surprisal setting the effective inspection cost.
Resolved arms exit the gate, fresh arms shut off above a prior-determined threshold, and selected linear-bandit overrides consume finite information budget.
Across Bernoulli bandits, linear bandits, and tabular MDPs, the same hyperparameters transfer without retuning, and DE shows much weaker regret growth than Thompson Sampling and $\varepsilon$-greedy in the tested unresolved regimes.
Delight improves acting for the same reason it improves learning: it prices scarce resources by the product of upside and surprisal.
\end{abstract}

\section{Introduction}
\label{sec:intro}

Most exploration theory studies the long-run regime, where the agent has enough interaction to resolve the optimal policy.
Optimism, posterior sampling, information-directed sampling, and Bayes-optimal planning all search broadly until uncertainty is resolved~\citep{lattimore2020bandit,russo2014ids}.
That broad search pays off when the budget is large enough to amortize it.
We target the unresolved regime: action spaces or horizons large enough that broad search cannot be amortized within the available budget.
In this regime, large-scale systems default to $\varepsilon$-greedy.
The reasons are practical: $\varepsilon$-greedy is additive, posterior-free, and simple to implement.
It follows a production policy most of the time and only overrides with probability $\varepsilon$, so the worst-case disruption is bounded.
The problem is not the override rate, but how the override is spent.
$\varepsilon$-greedy spends it blindly, often on actions whose posterior upside is already negligible.

We propose \emph{Delight-gated exploration} (DE), which keeps the same additive host-override structure but replaces the blind override with a targeted one.
DE uses posterior information to filter candidates.
On override rounds, it selects only actions whose \emph{prospective delight} (the product of expected improvement and surprisal) exceeds a gate price.
Actions without enough upside do not pass the gate.
A fixed gate price therefore implements satisficing search.
The agent explores while posterior upside exceeds the price of deviation, not until all uncertainty is resolved.
Like $\varepsilon$-greedy, DE is a mixture of host and override.
The difference is that the override is targeted rather than blind.
This paper transposes the Delightful Policy Gradient~\citep{osband2026delightfulpolicygradient} from learning to acting.
There, delight gates which updates to apply; here, prospective delight gates which actions to try.

The distinction is AND versus OR.
Standard optimism keeps actions exploration-eligible when uncertainty remains high.
DE requires both posterior upside and surprisal.
When the environment cannot be fully resolved, this difference becomes decisive.
The optimistic set remains broad because many actions are still uncertain, while the delight-gated set shrinks as the host improves.

The same gate is also a reservation-value rule.
In Weitzman's Pandora problem, inspecting an action costs $c$ and reveals its latent value.
The optimal policy inspects only while the current best value is below a reservation threshold~\citep{weitzman1979optimal}.
DE recovers the same eligibility condition, with the gate price and surprisal jointly setting the effective inspection cost.
For fresh $\mathrm{Beta}(1,1)$ Bernoulli arms, this gives a shutoff threshold above which no untried arm passes the gate.
DE explores not because an action is uncertain, but because resolving that uncertainty is worth its price.

We make four contributions.
First, we introduce DE and show strong empirical scaling across Bernoulli bandits, linear bandits, and tabular MDPs, with the same hyperparameters transferred across all three (Section~\ref{sec:experiments}).
Second, we connect DE to optimal search theory: the gate is Pandora eligibility, and horizon-priced reservation search achieves the optimal prior-tail rate in a revealed-value discovery model (Sections~\ref{sec:reservation} and~\ref{sec:tail_scaling}).
Third, we give structural analysis explaining the observed scaling: resolved arms exit the gate, fresh arms shut off at a reservation threshold, and selected linear-bandit overrides consume finite information budget (Section~\ref{sec:analysis}).
Fourth, we state the limits clearly: fixed-price DE is a finite-budget heuristic, and a full regret bound for the greedy-host noisy-bandit version remains open (Section~\ref{sec:limits}).

\section{Algorithm}
\label{sec:algorithm}

Delight-gated exploration augments a greedy host with a sparse exploratory override.
With probability $1-\varepsilon_t$ the agent follows the host.
With probability $\varepsilon_t$ it draws from the override.
The only question is how to score candidate exploratory actions.
Our answer mirrors the Delightful Policy Gradient: use the product of upside and novelty.
For exploration, upside is measured by expected improvement and novelty by surprisal under the host policy.

\subsection{Definition}
\label{sec:definition}

Consider a finite action set $\Ac$ with unknown expected rewards $f(a)$.
Let $m_a(t) = \E[f(a)\mid\Hc_t]$ be the posterior mean reward given history $\Hc_t$.
The host baseline is $v_t = \max_a m_a(t)$.
We write $\pi_t^{\mathrm{host}}$ for the host policy.
The analysis below studies the greedy limit ($\pi_t^{\mathrm{host}}$ plays $\argmax_a m_a(t)$), while experiments use a near-greedy Boltzmann host.

We measure upside by expected improvement over the host, $\EI_t(a) = \E[(f(a)-v_t)^+\mid\Hc_t]$.

Novelty is measured by a capped relative surprisal $\surp_t(a) = \min\{[-\log \pi_t^{\mathrm{host}}(a)-\ell_{\min,t}]_+,\,L\}$, where $\ell_{\min,t} = \min_b\{-\log \pi_t^{\mathrm{host}}(b)\}$.
This makes the host's preferred action have zero surprisal.
For a greedy host, non-host actions saturate at $L$, and DE reduces to a thresholded EI rule.

Prospective delight is the product $\tilde{\delight}_t(a) = \EI_t(a)\surp_t(a)$.
This transposes the principle of the Delightful Policy Gradient family~\citep{osband2026delightfulpolicygradient,osband2026doesgradientsparkjoy} from learning to acting.
Prospective delight must exceed $\lambda$ for an exploratory override to fire.

The gated set collects actions with $\tilde{\delight}_t(a)\ge\lambda$.
If non-empty, the override samples $q_t^\lambda(a)\propto \tilde{\delight}_t(a)\one\{a\in\Gc_t\}$; otherwise it defaults to the host.
The final acting policy is
\begin{equation}
\label{eq:acting}
\pi_t^{\mathrm{act}} = (1-\varepsilon_t)\pi_t^{\mathrm{host}} + \varepsilon_t q_t^\lambda,
\end{equation}
where $\varepsilon_t = M/(M+t)$ is annealed with half-life $M$.
The gate price $\lambda$ is measured in reward units times surprisal; across tasks we use rewards or value estimates on comparable unit scales.

\subsection{A Reservation-Value View}
\label{sec:reservation}

The delight gate has a classical search interpretation.
Consider a revealed-value search problem.
Each action $a$ has a latent value $X_a$ drawn from a known prior.
Inspecting $a$ reveals $X_a$ and costs $c_a$, and the decision maker keeps the best revealed value.
Pandora's rule assigns each action a reservation value.
When the reservation equation has a unique solution, this value $z_a$ satisfies
\begin{equation}
\label{eq:pandora_reservation}
  \E[(X_a-z_a)^+] = c_a .
\end{equation}
When the solution is not unique, the generalized inverse below gives the same stopping rule.
Pandora opens actions in decreasing order of reservation value and stops once the current best value exceeds every unopened reservation value~\citep{weitzman1979optimal}.

DE has the same reservation structure.
For an action with host-relative surprisal $\surp_a>0$, define the effective inspection cost $c_a := \lambda/\surp_a$.
When $\surp_a=0$, we interpret the effective inspection cost as infinite: such actions are handled by the host rather than the exploratory override.
Then action $a$ passes the delight gate at baseline $v$ iff
\[
  \surp_a\E[(X_a-v)^+] \ge \lambda
  \quad\Longleftrightarrow\quad
  \E[(X_a-v)^+] \ge c_a.
\]
Define the delight reservation index
\begin{equation}
\label{eq:delight_reservation}
  z_a^\lambda
  :=
  \sup\left\{
    z:\surp_a\,\E[(X_a-z)^+] \ge \lambda
  \right\}.
\end{equation}

\begin{proposition}[The DE gate is Pandora eligibility]
\label{prop:pandora_gate}
In the revealed-value search model with fixed $\surp_a>0$, let $\Gc(v)$ denote the DE gated set at baseline $v$.
Then
\[
  \Gc(v)
  =
  \{a:\surp_a\E[(X_a-v)^+]\ge\lambda\}
  =
  \{a:z_a^\lambda\ge v\}.
\]
In the unthrottled revealed-value search model, if the override always selects the gated action with largest $z_a^\lambda$, the resulting inspection order recovers Pandora's ordering rule.
\end{proposition}

\begin{proof}
Let $g_a(z)=\E[(X_a-z)^+]$.
This function is continuous and nonincreasing in $z$ for integrable $X_a$.
Since $g_a$ is continuous and nonincreasing, the set $\{z:\surp_a g_a(z)\ge\lambda\}$ is a closed lower interval, possibly empty.
Its right endpoint is $z_a^\lambda$ under the convention that the supremum of the empty set is $-\infty$.
Therefore $\surp_a g_a(v)\ge\lambda$ exactly when $v\le z_a^\lambda$.
This proves the gated-set identity.
Weitzman's rule opens an eligible box with maximal reservation value and stops when no reservation value exceeds the current best value.
Thus DE recovers Pandora's eligibility condition; with an unthrottled max-reservation override, it also recovers Pandora's ordering rule.
\end{proof}

For fresh arms sharing a common prior $F$ and capped surprisal $L$, define the fresh-arm reservation threshold by the same generalized inverse:
\begin{equation}
\label{eq:general_reservation}
  v_\lambda
  :=
  \sup\left\{
    v:\E_{X\sim F}[(X-v)^+] \ge \lambda/L
  \right\}.
\end{equation}
When the equation has a unique solution, this is equivalently
$\E_{X\sim F}[(X-v_\lambda)^+]=\lambda/L$.
Different priors induce different reservation thresholds, making the gate price interpretable in units of prior tail value.

The equivalence above is exact for revealed-value search.
In noisy bandits, a pull does not reveal the latent mean $f(a)$.
There, DE uses posterior expected improvement as a value-of-perfect-information proxy.
Appendix~\ref{app:kg_ei} shows that, for independent-arm posteriors, one-step knowledge gradient is bounded above by expected improvement.

\subsection{Bernoulli Bandits}
\label{sec:bernoulli}

We specialize to $K$-armed Bernoulli bandits with independent $\mathrm{Beta}(1,1)$ priors.
After $n_a$ pulls with $S_a$ successes, the posterior is $\mathrm{Beta}(1+S_a,\,1+n_a-S_a)$.
Its mean is $m_a(t) = (1+S_a)/(2+n_a)$.

Expected improvement admits a closed form via the regularized incomplete beta function.
For analysis, a one-sided second-moment bound is sharper than a tail bound.
For any random variable $X$ with mean $m$ and variance $\sigma^2$, and any $v>m$,
\begin{equation}
\label{eq:ei_second_moment}
\E[(X-v)^+] \le \frac{\sigma^2}{v-m}.
\end{equation}
On the event $X\ge v$ we have $X-v\le X-m\le (X-m)^2/(v-m)$.
Taking expectations gives the claim.
The Beta posterior variance satisfies $\Var[f(a)\mid\Hc_t] \le 1/(4(n_a+3))$, so
\begin{equation}
\label{eq:beta_ei}
\EI_t(a) \le \frac{1}{4(n_a+3)(v_t - m_a(t))} \qquad\text{whenever }m_a(t)<v_t.
\end{equation}
Expected improvement decays with repeated pulls.
Once a suboptimal arm is sufficiently resolved, its prospective delight falls below $\lambda$ and it exits the gated set.

For untried arms, the reservation threshold~\eqref{eq:general_reservation} gives the shutoff condition directly.
With a $\mathrm{Beta}(1,1)$ prior, $X\sim\mathrm{Uniform}(0,1)$ and $\E[(X-v)^+]=(1-v)^2/2$.
For $0<\lambda<L/2$, solving $(1-v)^2/2=\lambda/L$ yields
\begin{equation}
\label{eq:fresharm}
v_{\mathrm{off}} = 1 - \sqrt{2\lambda/L}.
\end{equation}
For $\lambda\ge L/2$, the generalized threshold is nonpositive, so fresh arms do not pass the gate under the $\mathrm{Beta}(1,1)$ prior for baselines $v_t\ge 0$.
On any round where $v_t > v_{\mathrm{off}}$, no untried arm belongs to the gated set, regardless of how many arms remain.
Because the algorithm uses the current baseline $v_t$ rather than a running maximum, the fresh-arm gate can reopen if the host baseline drops after failures.

\begin{center}
\begin{minipage}{0.88\columnwidth}
\begin{algorithm}[H]
\caption{Delight-gated exploration}
\label{alg:explore}
\begin{algorithmic}[1]
\Require Posteriors $\{p_t(\cdot\mid a)\}$, price $\lambda$, cap $L$, schedule $\varepsilon_t$
\State $m_a(t) \gets \E[f(a)\mid\Hc_t]$ for all $a$; \quad $v_t \gets \max_a m_a(t)$
\State $\ell_{\min,t} \gets \min_b\{-\log \pi_t^{\mathrm{host}}(b)\}$
\For{each $a\in\Ac$}
    \State $\EI_t(a) \gets \E[(f(a)-v_t)^+\mid\Hc_t]$
    \State $\surp_t(a) \gets \min\{[-\log \pi_t^{\mathrm{host}}(a)-\ell_{\min,t}]_+,\,L\}$
    \State $\tilde{\delight}_t(a) \gets \EI_t(a)\surp_t(a)$
\EndFor
\State $\Gc_t \gets \{a : \tilde{\delight}_t(a) \ge \lambda\}$
\If{$\Gc_t \ne \emptyset$}
    \State $q_t^\lambda(a) \propto \tilde{\delight}_t(a)\one\{a\in\Gc_t\}$
\Else
    \State $q_t^\lambda \gets \pi_t^{\mathrm{host}}$
\EndIf
\State Draw $A_t \sim (1-\varepsilon_t)\pi_t^{\mathrm{host}} + \varepsilon_t q_t^\lambda$
\end{algorithmic}
\end{algorithm}
\end{minipage}
\end{center}

\section{Experiments}
\label{sec:experiments}

We test DE across three settings, each adding one axis of difficulty.
Bernoulli bandits test scaling with independent arms and exact posteriors.
Linear bandits add generalization: pulling one arm teaches the agent about others via shared features.
Tabular MDPs (DeepSea) add sequential structure: the agent must coordinate $H$ actions to reach a distant reward.
The same hyperparameters ($M{=}100$, $\lambda{=}0.1$, $L{=}10$) transfer across all three without retuning.
Because $\lambda$ has units of reward times surprisal, this transfer is meaningful only when rewards or value estimates are on comparable scales.
In all experiments, returns are approximately unit-scale.

\subsection{Bernoulli Bandits: Exact Posteriors}
\label{sec:exp_bernoulli}

We use $\varepsilon_t = M/(M+t)$ with $M{=}100$, $\lambda{=}0.1$, $L{=}10$.
Baselines are Thompson Sampling and annealed $\varepsilon$-greedy with the same schedule.
All curves show mean regret over 100 seeds with $\pm 1$ standard error.

Figure~\ref{fig:bernoulli_arms} shows the main scaling result.
As $K$ grows, DE's regret remains nearly flat while Thompson Sampling and $\varepsilon$-greedy degrade.
Section~\ref{sec:fresharms} explains this: when the host baseline exceeds the shutoff threshold, untried arms are excluded.

\begin{figure}[ht!]
\centering
\includegraphics[width=0.5\textwidth]{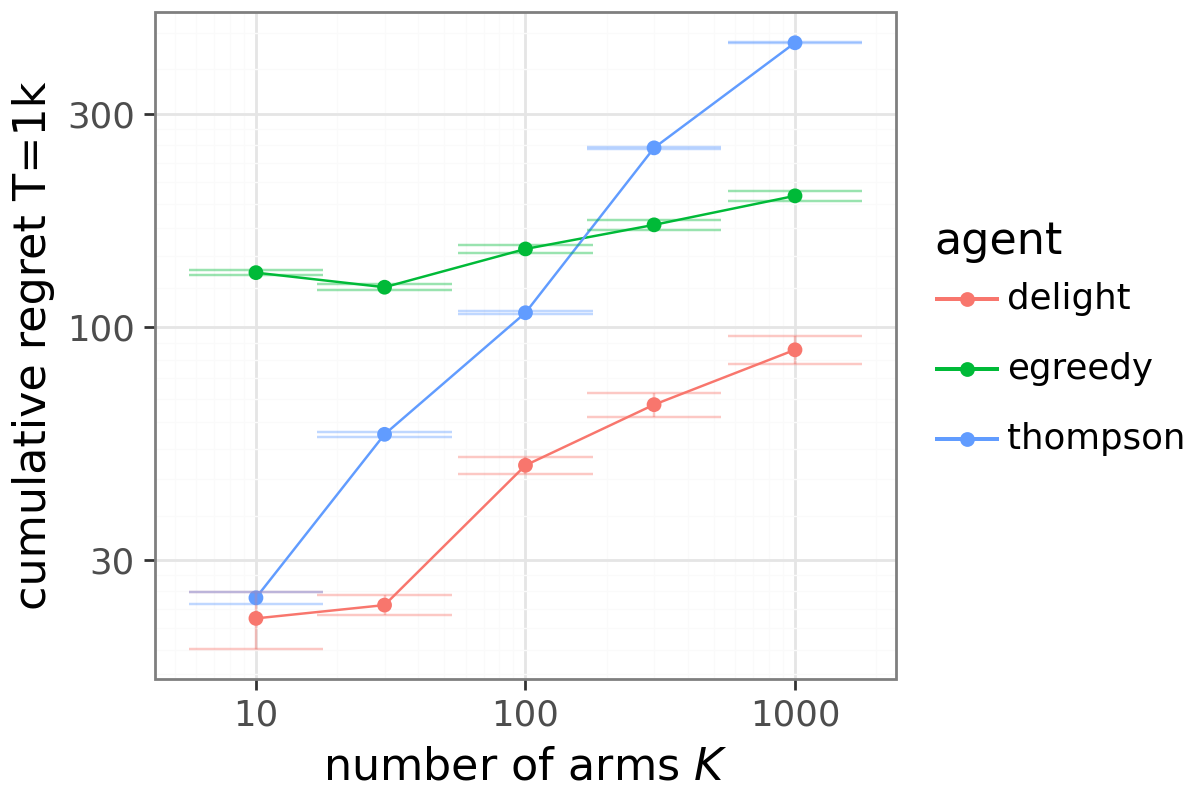}
\caption{Cumulative regret after $T{=}1000$ rounds as the number of arms $K$ increases.
DE remains nearly flat while Thompson Sampling and annealed $\varepsilon$-greedy degrade with $K$.}
\label{fig:bernoulli_arms}
\end{figure}
Figure~\ref{fig:bernoulli_facet} confirms the two regimes over time.
At $K=10$, the horizon is large relative to the environment and DE matches Thompson Sampling.
At $K=1000$, the environment cannot be resolved within budget, and the gap widens steadily as DE avoids wasted exploration.

\begin{figure}[ht!]
\centering
\includegraphics[width=0.95\textwidth]{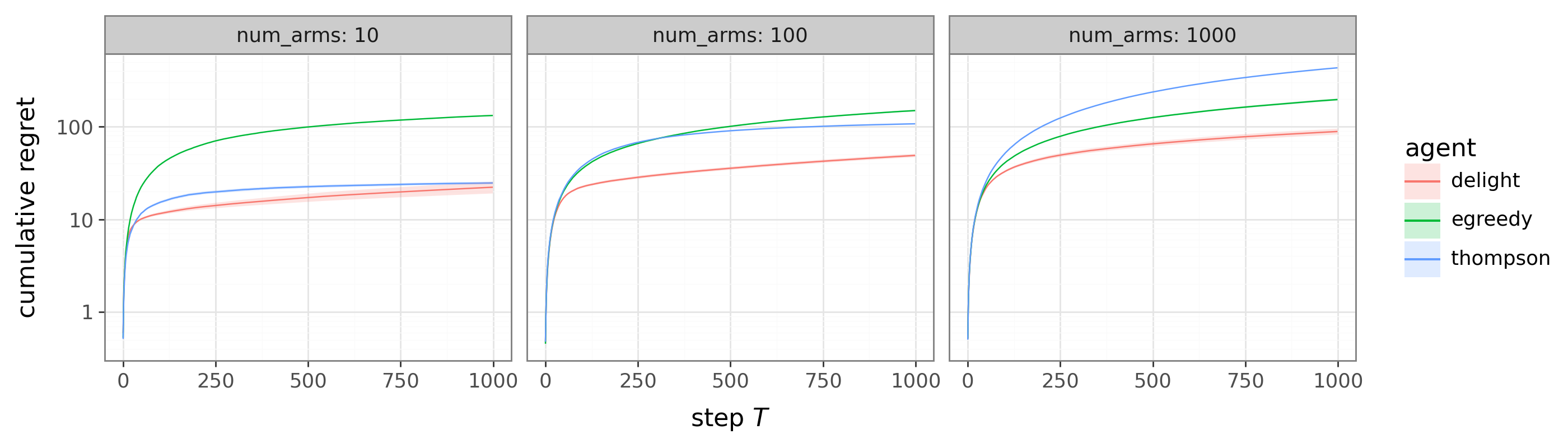}
\caption{Learning curves for $K \in \{10, 100, 1000\}$.
When the environment is small, DE matches Thompson Sampling.
As the number of arms grows, DE increasingly outperforms both baselines by shutting off wasted exploration.
Shaded regions show $\pm 1$ standard error over 100 seeds.}
\label{fig:bernoulli_facet}
\end{figure}

Appendix Figure~\ref{fig:bernoulli_tune} shows sensitivity to the annealing half-life $M$ and gate price $\lambda$.
DE outperforms Thompson Sampling over a broad range of both hyperparameters, indicating that the gains do not depend on delicate tuning.

\subsection{Linear Bandits: Transfer Without Retuning}
\label{sec:exp_linear}

We now test whether the scaling observation from independent arms carries over to a setting with generalization.
In a linear bandit, arms share a $d$-dimensional feature representation, so each pull teaches the agent about all arms with similar features.
We transfer the Bernoulli hyperparameters directly ($M=100$, $\lambda=0.1$), without retuning.
For the dimension sweep we fix $K=100$ and $\sigma=1.0$.
For the noise sweep we fix $K=100$ and $D=30$.

Figure~\ref{fig:linear_overall} shows that the same qualitative picture persists.
DE outperforms both baselines across the full sweep over the dimension $D$ and across the full range of noise levels.
The key point is not just that DE wins, but that the Bernoulli configuration transfers directly to the linear case without modification.
The gate captures a structural property of exploration, not a peculiarity of the exact posterior.

\begin{figure}[ht!]
\centering
\begin{subfigure}{0.48\textwidth}
  \centering
  \includegraphics[width=\linewidth]{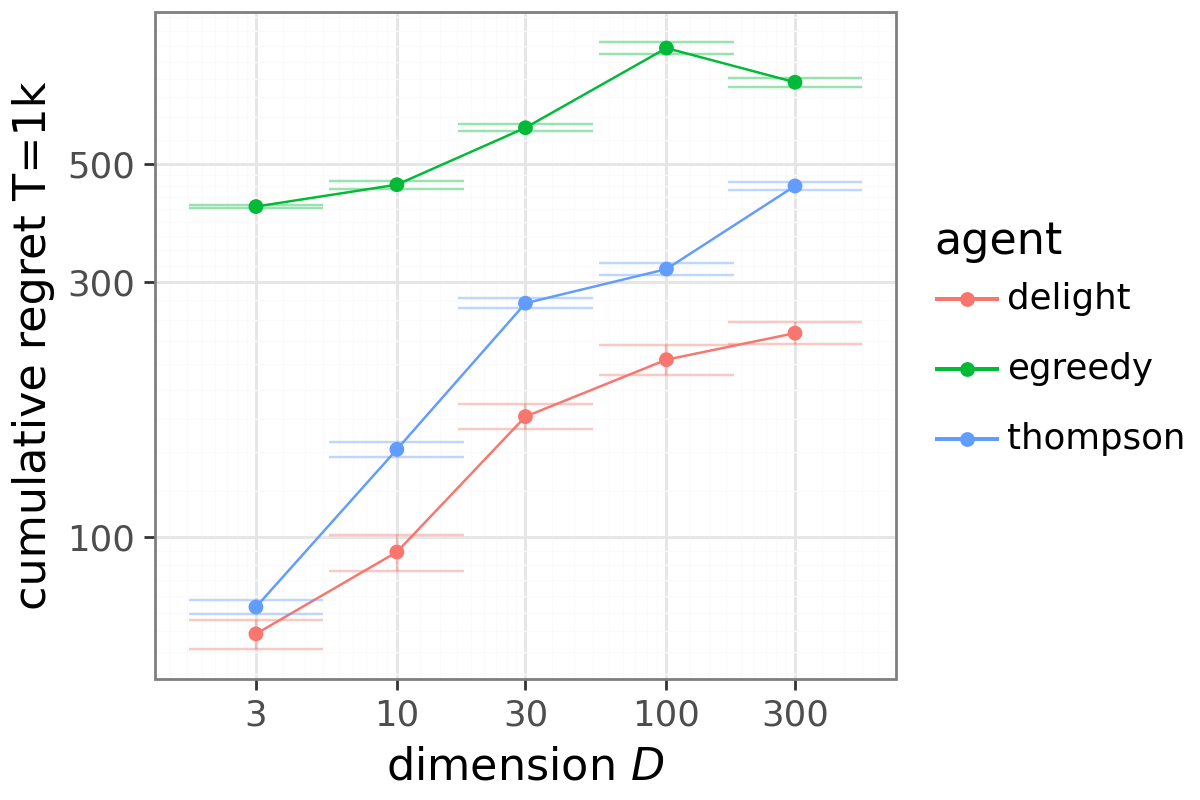}
  \caption{Regret vs.\ dimension $D$ ($K{=}100, \sigma{=}1.0$).}
  \label{fig:linear_dim}
\end{subfigure}
\hfill
\begin{subfigure}{0.48\textwidth}
  \centering
  \includegraphics[width=\linewidth]{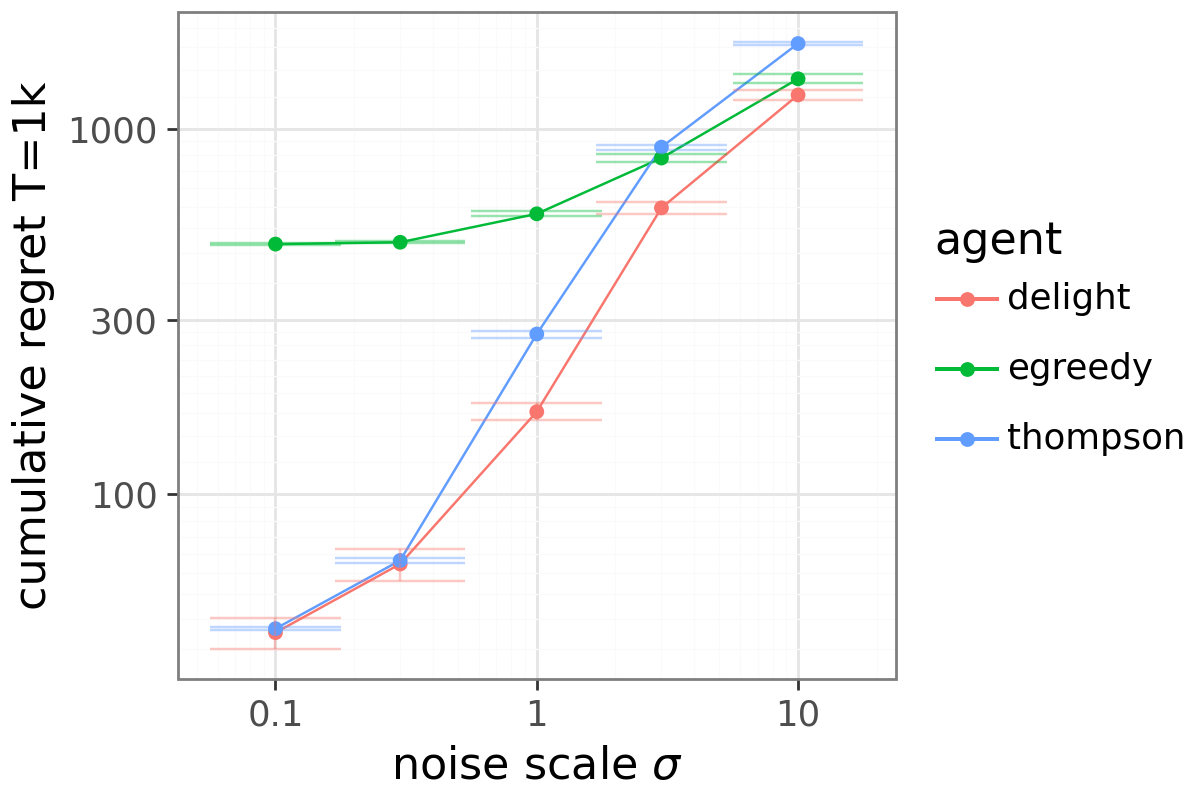}
  \caption{Regret vs.\ noise $\sigma$ ($D{=}30, K{=}100$).}
  \label{fig:linear_sigma}
\end{subfigure}
\caption{Linear bandits with no retuning from the Bernoulli setting.
Left: regret as the dimension $D$ increases at fixed $K{=}100$ and $\sigma{=}1.0$.
Right: regret as the noise scale $\sigma$ increases at fixed $D{=}30$ and $K{=}100$.
DE outperforms both baselines across both sweeps.}
\label{fig:linear_overall}
\end{figure}

The corresponding sensitivity analysis (Appendix, Figure~\ref{fig:linear_tune}) shows the same broad basin of good performance as in the Bernoulli setting.
DE continues to outperform Thompson Sampling across a wide range of $M$ and $\lambda$, supporting the view that the mechanism transfers rather than merely adapts.

\subsection{Tabular MDPs: DeepSea}
\label{sec:exp_deepsea}

We now test whether the delight gate extends to sequential problems with long-term consequences.
DeepSea is a tabular MDP of depth $H$ with binary actions at each state~\citep{osband2016deep,osband2020bsuite}.
It is a canonical test for deep exploration: only one specific sequence of $H$ actions reaches the reward at the far end of the grid, while all other paths yield zero.
A pure greedy policy never finds the reward.
$\varepsilon$-greedy requires $\Omega(2^H)$ episodes to find it by chance.

We apply DE per step, using a mean-based surrogate for expected improvement rather than exact Bayesian EI.
Exact sequential EI requires propagating full posterior variance through the Bellman equation, which is intractable for any nontrivial MDP; the surrogate preserves the gating structure while remaining tractable (see Appendix~\ref{app:exp_details} for details).
At the start of episode $e$, we freeze a posterior-mean planning model and compute values $Q^{\mathrm{plan}}_e$.
The host is a Boltzmann distribution over these frozen planning values,
\[
  \pi^{\mathrm{host}}_e(a\mid s)
  \propto
  \exp(Q^{\mathrm{plan}}_e(s,a)/\tau),
  \qquad \tau=0.01.
\]
During the episode, the posterior is updated online as transitions and rewards are observed, producing an updated posterior-mean evaluator $Q^{\mathrm{post}}_{e,t}$.
We measure sequential improvement relative to the frozen host baseline:
\[
  \EI_{e,t}(s,a)
  =
  \left(Q^{\mathrm{post}}_{e,t}(s,a)-V^{\mathrm{plan}}_e(s)\right)_+,
  \qquad
  V^{\mathrm{plan}}_e(s)
  =
  \sum_b \pi^{\mathrm{host}}_e(b\mid s)Q^{\mathrm{plan}}_e(s,b).
\]
Surprisal is computed under the frozen host:
\[
  \surp_e(s,a)
  =
  \min\{[-\log \pi^{\mathrm{host}}_e(a\mid s)-\ell_{\min,e}(s)]_+,L\}.
\]
This stale-plan formulation is important: new observations can create positive improvement for actions that remain surprising under the plan that would otherwise be deployed.
It also avoids the two-action same-snapshot degeneracy in which the host action has zero surprisal and the surprising action has nonpositive improvement.

We use the same gate price $\lambda$ and surprisal cap $L$ as in the Bernoulli setting, without retuning.
The agent maintains a Beta-Dirichlet posterior over rewards and transitions; true dynamics are unknown and must be learned.
PSRL~\citep{strens2000bayesian,osband2013more} uses the same posterior and planning code but samples a full MDP from the posterior each episode.
In this tabular benchmark, DE matches or outperforms PSRL across the tested problem sizes; we defer the full learning curves over time to Appendix~\ref{app:exp_details}.

\begin{figure}[ht!]
\centering
\includegraphics[width=0.5\textwidth]{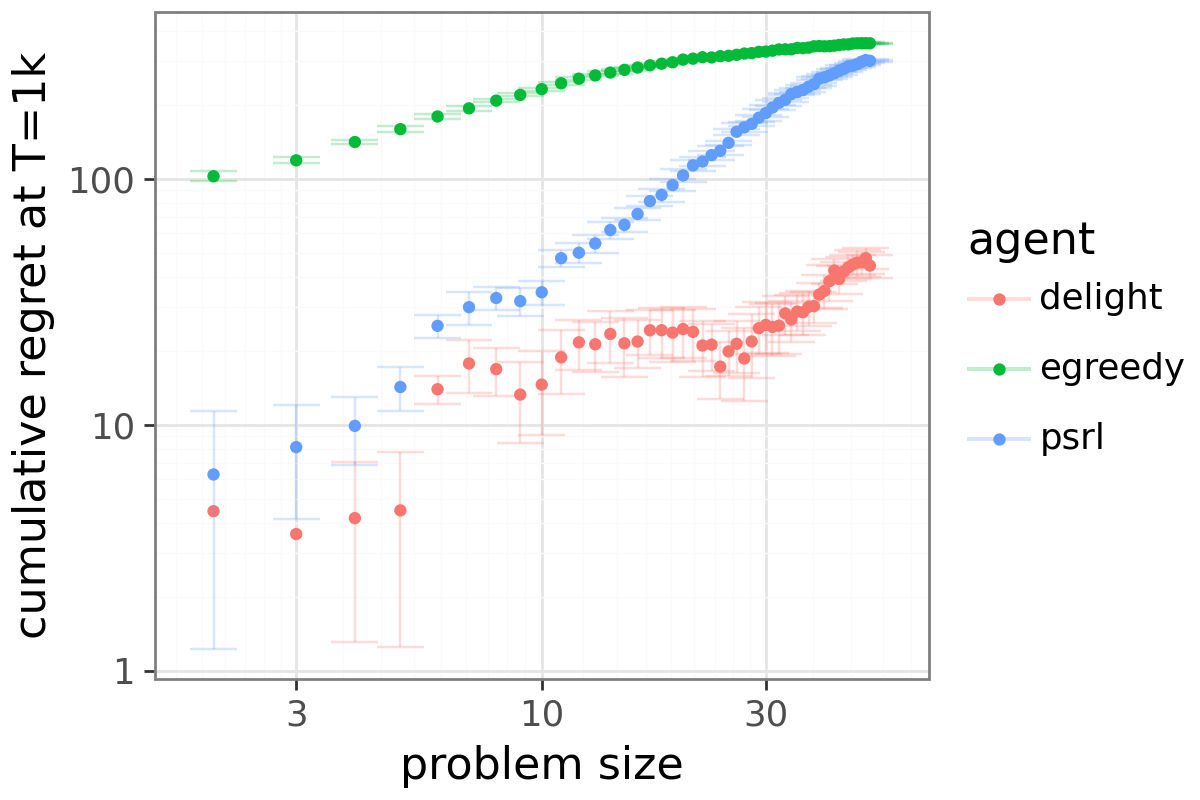}
\caption{Regret at $T{=}1000$ vs problem size $H$ for DeepSea. DE maintains lower regret than PSRL in this sweep, while $\varepsilon$-greedy degrades rapidly.}
\label{fig:deepsea_regret}
\end{figure}

\section{Why It Works: Structural Analysis}
\label{sec:analysis}

The experiments show three consistent patterns: DE matches classical methods in small worlds, scales gracefully in large worlds, and outperforms $\varepsilon$-greedy throughout.
This section provides structural analysis that explains each of these observations.

All structural results in this section analyze the deterministic greedy-host limit of Algorithm~\ref{alg:explore}.
In that limit, every non-host action has surprisal $L$, so the analyzed rule is effectively a thresholded expected-improvement rule.
The stochastic-host version used in the experiments is a nearby heuristic whose full product structure is active because surprisal varies across actions.

\subsection{Resolved Arms Drop Out}
\label{sec:small_bernoulli}

We formalize the Bernoulli setting with a greedy host and delight gate.

\begin{assumption}[Bernoulli bandit with delight gate]
\label{assump:main}
We consider a $K$-armed Bernoulli bandit with unknown means $\mu_a \in [0,1]$ and independent $\mathrm{Beta}(1,1)$ priors.
The host is greedy with respect to posterior means.
The override uses schedule $\varepsilon$, gate price $\lambda>0$, and surprisal cap $L>0$.
We assume a cold start in which each arm is pulled once, contributing an additive $O(K)$ term.
Let $a^\star$ denote an optimal arm and $\Delta_a := \mu_{a^\star} - \mu_a$ the suboptimality gap.
\end{assumption}

The gate mechanism admits two structural facts, proved in Appendix~\ref{app:proofs_small}.
Once both a suboptimal arm and the optimal arm have been pulled enough times, the posterior means separate and the greedy host permanently identifies the best arm.

\begin{lemma}[Separation]
\label{lem:separation_main}
On the concentration event $\mathcal{E}_\delta$ (probability $\ge 1-\delta$), if both $n_a(t) \ge N_a$ and $n_{a^\star}(t) \ge N_a$ where $N_a = \lceil 64\Delta_a^{-2}\log(2KT/\delta)\rceil$, then $m_{a^\star}(t) - m_a(t) \ge \Delta_a/2$.
\end{lemma}

After separation, the delight gate ensures that DE selects any given suboptimal arm only finitely many more times.

\begin{lemma}[Post-separation selected-pull budget]
\label{lem:gate_main}
After separation, the number of subsequent rounds on which DE selects arm $a$ via the gated override is at most
$N_a^{\mathrm{gate}} = \lceil L/(2\lambda\Delta_a)\rceil$.
\end{lemma}

Together, these show that once a suboptimal arm is sufficiently resolved, the greedy host stops selecting it and the remaining number of override selections is finite.
The per-arm selected-pull budgets depend on the gap $\Delta_a$ but not on the number of arms $K$.

\subsection{Linear Bandits: Selected-Override Budget}
\label{sec:small_linear}

The same gating mechanism extends to structured posteriors, such as linear bandits~\citep{dani2008stochastic,abbasi2011improved}.
When arms share a $d$-dimensional feature representation, pulling one arm teaches the agent about many others.
The relevant question is no longer how many arms exist, but how much posterior uncertainty the gate can spend.

The result is an information-budget bound.
Gate-open selection implies $\EI_t(a)\ge\lambda/L$, while Gaussian expected improvement is at most the posterior standard deviation.
The elliptical potential lemma then bounds selected gated overrides by the linear information gain, independent of the number of arms.
Appendix~\ref{app:proofs_linear} gives the formal statement and proof.

\subsection{Fresh Arms Drop Out}
\label{sec:fresharms}

We now consider the large-world setting without forced initialization.

\begin{assumption}[Large-world Bernoulli setting]
\label{assump:large}
We consider a $K$-armed Bernoulli bandit with unknown means $\mu_a \in [0,1]$ and independent $\mathrm{Beta}(1,1)$ priors.
The host is greedy with respect to posterior means.
The override uses schedule $\varepsilon_t$, gate price $\lambda > 0$, and surprisal cap $L > 0$.
There is no forced initialization: arms may have zero pulls.
Let $a^\star$ denote an optimal arm and $\Delta_a := \mu_{a^\star} - \mu_a$.
\end{assumption}

The key structural question is whether fresh arms remain exploration-eligible independently of $K$.
For tried actions that have separated from the optimal arm, Section~\ref{sec:small_bernoulli} gives a finite selected-pull budget.
For untried actions, the gate criterion depends only on the current host baseline $v_t$.

\begin{proposition}[Pointwise fresh-arm shutoff]
\label{prop:budget}
Under Assumption~\ref{assump:large}, fix any round $t$.
Every untried arm $a$ satisfies
$\tilde{\delight}_t(a) \le L(1-v_t)^2/2$.
Hence if $v_t > v_{\mathrm{off}} := 1 - \sqrt{2\lambda/L}$, then no untried arm belongs to the gated set $\Gc_t$.
Moreover, if a tried arm $a$ has separated from the optimal arm, then its remaining number of selected override pulls is at most
$N_a^{\mathrm{gate}} = \lceil L/(2\lambda\Delta_a) \rceil$.
\end{proposition}

Because the algorithm uses $v_t$ rather than a running maximum, this exclusion is pointwise.
The fresh-arm gate can reopen if the host baseline drops.
This shutoff is qualitatively different from unpriced optimism, which keeps unresolved arms exploration-eligible while their uncertainty remains high.

\begin{theorem}[Fresh-arm shutoff versus optimism]
\label{thm:thermostat}
Under Assumption~\ref{assump:large}, fix $\lambda$ and $L$.
If $K \ge T$ and a UCB rule assigns $+\infty$ index to unpulled arms, it selects a previously untried arm on every round.
In contrast, under DE, on any round $t$ with $v_t > v_{\mathrm{off}}$, the gated set contains no untried arm.
\end{theorem}

The distinction is again AND versus OR.
Optimistic methods keep actions eligible when uncertainty alone is high, so the exploratory set remains broad whenever many actions are unresolved.
DE keeps actions override-eligible only when both posterior upside and surprisal are high enough.
When the host baseline is high enough, no untried arm has enough of both, so the gated set excludes them all on that round.

\subsection{Large-World Reservation Rates}
\label{sec:tail_scaling}

The reservation view also suggests how to choose the price from the horizon.
The cleanest setting is an infinite reservoir of fresh arms.
Each fresh arm has a revealed value $X\sim F$ supported on $[0,1]$.
Regret is measured relative to the upper endpoint:
$R_T = \E\sum_{t=1}^T (1-X_{A_t})$.
Let $p(y):=\Pr(X\ge 1-y)$.
A reservation policy with threshold $1-y$ samples fresh arms until one exceeds $1-y$, then exploits the best observed arm.
Its regret is bounded by
\begin{equation}
\label{eq:threshold_regret}
  R_T(y) \le \frac{1}{p(y)} + Ty.
\end{equation}
The corresponding DE price is
\begin{equation}
\label{eq:tail_price}
  \lambda_y = L\,\E[(X-(1-y))^+] = L\int_0^y p(s)\,ds.
\end{equation}

\begin{theorem}[Prior-tail rate for horizon-priced reservation search]
\label{thm:tail_rate}
Assume the upper tail of $F$ is polynomial: for some $\alpha>0$ and constants $0<c_-\le c_+<\infty$,
$c_-y^\alpha \le \Pr(X\ge 1-y) \le c_+y^\alpha$
for all sufficiently small $y$.
Then the reservation policy with $y_T \asymp T^{-1/(\alpha+1)}$ achieves
$R_T = O(T^{\alpha/(\alpha+1)})$.
No policy in the unthrottled infinite-arm discovery model improves this rate in order.
The corresponding DE price obeys $\lambda_T = \Theta(L/T)$.
With constant override rate $\varepsilon$, the same upper-bound argument gives
$R_T = O(\varepsilon^{-1/(\alpha+1)}T^{\alpha/(\alpha+1)})$
and price $\lambda_T = \Theta(L/(\varepsilon T))$.
\end{theorem}

\noindent\textit{Proof sketch.}
The upper bound optimizes $1/p(y)+Ty$.
The lower bound is the matching search tradeoff: either a policy samples $\Omega(1/p(y))$ fresh arms, or it misses a $1-y$ arm with constant probability and pays $\Omega(Ty)$ regret.
Appendix~\ref{app:tail_scaling} gives the proof.

For $X\sim\mathrm{Uniform}(0,1)$, the unthrottled horizon price is $\lambda_T=L/(2T)$ and the optimal rate is $\Theta(\sqrt{T})$.

This theorem calibrates prices when inspections are available.
The host--override algorithm has two knobs: $\lambda$ controls which arms are eligible, and $\varepsilon_t$ controls how often eligible arms are inspected.
For tail exponent $\alpha$, the prior-tail rate requires $\Theta(T^{\alpha/(\alpha+1)})$ fresh inspections.
The default schedule gives only $O(M\log T)$ override opportunities asymptotically.
Thus the fixed-price experiments should be read as satisficing search, not asymptotic search.

\subsection{Fixed-History Advantage over $\varepsilon$-Greedy}
\label{sec:dominance}

DE and $\varepsilon$-greedy share the same host, schedule, and override probability.
They differ only in the override distribution.
When the gate is empty, DE plays the host while $\varepsilon$-greedy draws uniformly.
When the gate is open, DE draws only from actions with $\EI_t(a)\ge\lambda/L$.
As arms are resolved, their expected improvement collapses.
$\varepsilon$-greedy continues to average over them; DE excludes them from the gate.
Appendix~\ref{app:fixed_history} gives the fixed-history comparison.

The gate is not cosmetic.
Appendix~\ref{app:necessity} gives simple examples showing that pure greedy can lock in, persistent unpriced overrides can bleed regret, and fixed-temperature hosts can fail to commit.
The host must be cold enough to commit, and the override must be priced enough to shut off.

\section{Limitations and Open Problems}
\label{sec:limits}

The structural results bound the gate, not the host's regret.
For rewards in $[0,1]$, if DE is paired with a host whose guarantee is stable under the additional observations induced by the override, then the extra regret from override actions is at most $\sum_{t=1}^T\varepsilon_t=O(M\log(1+T/M))$.
Formalizing this inheritance requires a coupling argument.

The greedy-host variant remains open.
Because DE uses the current baseline $v_t$ rather than a running maximum, the gate can reopen.
We do not currently have a worst-case regret bound or a matching impossibility result.
Empirically, lock-in appears rare under i.i.d.\ $\mathrm{Uniform}(0,1)$ priors, but harder priors can shut off exploration prematurely.
This is the central tradeoff of fixed-price search.

The Pandora equivalence is exact only in the revealed-value model.
In noisy bandits, DE uses expected improvement as a value-of-perfect-information proxy.
For independent arms, this proxy upper bounds one-step knowledge gradient (Appendix~\ref{app:kg_ei}); for correlated posteriors and linear bandits, the relationship is more subtle.

A fixed price implements satisficing search.
A horizon-dependent price implements Bayes-search calibration.
In the reservoir model, the regret-oriented price scales as $\Theta(L/T)$ without throttling and $\Theta(L/(\varepsilon T))$ with constant override rate $\varepsilon$.
That price alone is not enough.
The prior-tail rate also requires enough inspection opportunities.
The default schedule supplies only $O(M\log T)$ override opportunities asymptotically, so the theorem calibrates prices rather than proving a guarantee for the fixed experimental schedule.

Algorithm~\ref{alg:explore} scores every action.
In large or continuous action spaces, the gate can be applied to a finite candidate set proposed by the host, a posterior sampler, or a retrieval system.
The structural results here analyze the finite-candidate case.

\section{Related Work}
\label{sec:related}

Classical bandit algorithms and their MDP counterparts explore broadly until uncertainty is resolved~\citep{auer2002finite,thompson1933likelihood,agrawal2012analysis,russo2014ids,jaksch2010near,strens2000bayesian,osband2013more}.
DE targets the unresolved regime where broad search cannot be amortized.

Satisficing Thompson sampling relaxes the target when exact optimality is too expensive~\citep{russo2022satisficing}.
Quantile regret and continuum-armed bandits study related regimes where full resolution is impossible~\citep{lattimore2024quantile,bubeck2011xarmed}.
DE keeps the original reward objective but prices exploratory deviation from the host.
Algorithmically, DE is closest to $\varepsilon$-greedy, but it restricts the override to actions whose prospective delight exceeds the gate price.
Expected improvement is usually a standalone acquisition function in Bayesian optimization~\citep{mockus1978application,jones1998efficient}.
DE instead uses EI as a priced eligibility test inside a host--override policy, with fallback to the host when no action clears the gate.
This differs from knowledge-gradient methods~\citep{frazier2008knowledge}, which choose the action maximizing one-step value of information.
DE uses EI as a cheaper proxy and bounds its total expenditure through the gate price.

Weitzman's Pandora rule gives the optimal policy for costly revealed-value search~\citep{weitzman1979optimal}.
DE recovers Pandora eligibility with effective cost $c=\lambda/\surp$.
The exact Pandora rule orders revealed-value inspections by reservation value.
DE uses the same stopping logic inside a host--override policy with noisy observations, using posterior EI as a value-of-perfect-information proxy.

This paper is the acting-side mirror of the Delightful Policy Gradient family~\citep{osband2026delightfulpolicygradient,osband2026doesgradientsparkjoy,osband2026delightfuldistributedpolicygradient}.
The shared principle is to allocate a scarce resource to beneficial surprise: plasticity on the learning side, exploratory pulls on the acting side.

\section{Conclusion}
\label{sec:conclusion}

DE replaces the blind $\varepsilon$-greedy override with a priced gate on prospective delight.
The gate is not arbitrary: it is the Pandora reservation condition for costly search, with surprisal setting the effective inspection cost.
A fixed price gives satisficing exploration; a horizon price gives the optimal prior-tail rate in the revealed-value discovery model.
Experiments across Bernoulli bandits, linear bandits, and tabular MDPs show that DE scales well while Thompson Sampling and $\varepsilon$-greedy degrade, with the same hyperparameters across all three.

A full regret bound for the greedy host remains open; the experiments point toward Bayes regret as the natural formal target.
The broader lesson is price before breadth.
In unresolved regimes, uncertainty alone is not a reason to explore; uncertainty with priced upside is.

\newpage
\bibliography{references}
\bibliographystyle{plainnat}

\newpage
\appendix

\section{Bernoulli Technical Lemmas}
\label{app:lemmas}

This section collects the technical lemmas used in the Bernoulli proofs for Sections~\ref{app:proofs_small} and~\ref{app:proofs_large}.
Throughout, we work under Assumption~\ref{assump:main} for small-world results and Assumption~\ref{assump:large} for large-world results.

The delight score of an untried arm is controlled by the current host baseline $v_t$.

\begin{lemma}[Untried-arm delight]
\label{lem:untried}
Under a $\mathrm{Beta}(1,1)$ prior, an untried arm satisfies $\EI_t(a) = (1-v_t)^2/2$ and $\surp_t(a) \le L$.
Hence $\tilde{\delight}_t(a) \le L(1-v_t)^2/2$.
In particular, if $v_t > v_{\mathrm{off}} := 1 - \sqrt{2\lambda/L}$, then no untried arm lies in $\Gc_t$.
\end{lemma}

\begin{proof}
Since the arm is untried, its posterior is the prior $\mathrm{Beta}(1,1) = \mathrm{Uniform}(0,1)$.
For $X \sim \mathrm{Uniform}(0,1)$: $\EI_t(a) = \E[(X-v_t)^+] = \int_{v_t}^1 (x-v_t)\,dx = (1-v_t)^2/2$.
Under a greedy host, $\surp_t(a) \le L$, so $\tilde{\delight}_t(a) \le L(1-v_t)^2/2$.
If $v_t > v_{\mathrm{off}}$, then $L(1-v_t)^2/2 < \lambda$, so $\tilde{\delight}_t(a) < \lambda$ for every untried arm.
\end{proof}

The Beta posterior mean is close to the empirical mean.
The gap is exactly computable.

\begin{lemma}[Posterior--empirical deviation]
\label{lem:shrinkage}
For a $\mathrm{Beta}(1+S, 1+n-S)$ posterior with $n$ observations and $S$ successes:
\[
  |m_a(t) - \hat\mu_a(t)| \le \frac{1}{n+2}.
\]
\end{lemma}

\begin{proof}
When $n = 0$: $m = 1/2$ and $\hat\mu = 0$, so the bound is $1/2$.
When $n \ge 1$: $m - \hat\mu = (n - 2S)/(n(n+2))$.
Since $0 \le S \le n$, we have $|n - 2S| \le n$, so $|m - \hat\mu| \le 1/(n+2)$.
\end{proof}

Empirical means concentrate around true means uniformly over all arms and all rounds.

\begin{lemma}[Uniform concentration]
\label{lem:concentration}
Let $\hat\mu_{a,n}$ be the empirical mean of the first $n$ observations from arm $a$.
Define the good event
\[
  \mathcal{E}_\delta
  \;:=\;
  \bigcap_{a=1}^K \bigcap_{n=1}^T
  \left\{
    |\hat\mu_{a,n} - \mu_a|
    \le
    \sqrt{\frac{\log(2KT/\delta)}{2n}}
  \right\}.
\]
Then $\Pr(\mathcal{E}_\delta) \ge 1 - \delta$.
\end{lemma}

\begin{proof}
Hoeffding's inequality gives failure probability at most $\delta/(KT)$ per $(a,n)$ pair.
A union bound over all $KT$ pairs gives $\Pr(\mathcal{E}_\delta^c) \le \delta$.
For any round $t$, $\hat\mu_a(t) = \hat\mu_{a,n_a(t)}$, so the bound holds simultaneously for all arms and rounds.
\end{proof}

After enough data, the greedy host permanently identifies the best arm.

\begin{lemma}[Arm separation]
\label{lem:separation}
Let $N_a := \lceil 64\Delta_a^{-2}\log(2KT/\delta)\rceil$.
On $\mathcal{E}_\delta$, if $n_a(t) \ge N_a$ and $n_{a^\star}(t) \ge N_a$, then $m_{a^\star}(t) - m_a(t) \ge \Delta_a/2$.
\end{lemma}

\begin{proof}
The argument combines concentration, shrinkage, and the triangle inequality.

\paragraph{Concentration.}
On $\mathcal{E}_\delta$ with $n \ge N_a$:
by substitution, $|\hat\mu_a - \mu_a| \le \Delta_a/8$.

\paragraph{Shrinkage.}
By Lemma~\ref{lem:shrinkage}: $|m_a - \hat\mu_a| \le 1/(N_a + 2) \le \Delta_a/8$.

\paragraph{Combine.}
The triangle inequality gives $|m_a - \mu_a| \le \Delta_a/4$ for both $a$ and $a^\star$.
Then $m_{a^\star} - m_a \ge (\mu_{a^\star} - \mu_a) - \Delta_a/2 = \Delta_a/2$.
\end{proof}

After separation, the expected improvement of a suboptimal arm decays with each pull.

\begin{lemma}[EI shut-off via second-moment bound]
\label{lem:ei_shutoff}
If arm $a$ is separated ($v_t - m_a(t) \ge \Delta_a/2$), then:
\[
  \EI_t(a) \;\le\; \frac{1}{2(n_a(t)+3)\,\Delta_a}.
\]
\end{lemma}

\begin{proof}
For any random variable $X$ with mean $m$ and variance $\sigma^2$, and $v > m$:
on the event $X \ge v$, $X - v \le X - m \le (X-m)^2/(v-m)$.
Taking expectations gives $\E[(X-v)^+] \le \sigma^2/(v-m)$.
The Beta posterior variance satisfies $\sigma^2 \le 1/(4(n+3))$.
Under separation, $v_t - m_a(t) \ge \Delta_a/2$, so $\EI_t(a) \le 1/(2(n_a(t)+3)\Delta_a)$.
\end{proof}

Once an arm is separated, DE can select it via the override only finitely many more times, with a count independent of $K$.

\begin{lemma}[Selected-pull shut-off count]
\label{lem:gate_count}
In the separated regime, any round that selects arm $a$ via the gated override must satisfy $n_a(t) \le L/(2\lambda\Delta_a) - 3$.
Consequently, the number of post-separation override selections of arm $a$ is at most $N_a^{\mathrm{gate}} := \lceil L/(2\lambda\Delta_a)\rceil$.
\end{lemma}

\begin{proof}
The gate requires $\EI_t(a) \cdot \surp_t(a) \ge \lambda$.
Since $\surp_t(a) \le L$, a necessary condition is $\EI_t(a) \ge \lambda/L$.
By Lemma~\ref{lem:ei_shutoff}, this fails once $n_a(t) > L/(2\lambda\Delta_a) - 3$.
\end{proof}

\section{Proofs for Section~\ref{sec:small_bernoulli}: Bernoulli Bandits}
\label{app:proofs_small}

We prove the two structural lemmas stated in Section~\ref{sec:small_bernoulli}.
Throughout, $m_a(t)$ is the posterior mean of arm $a$, $v_t=\max_a m_a(t)$ is the host baseline, and $\Delta_a=\mu_{a^\star}-\mu_a$ is the suboptimality gap.
Both proofs work on the good event $\mathcal{E}_\delta$ from Lemma~\ref{lem:concentration}.

\begin{proof}[Proof of Lemma~\ref{lem:separation_main}]
On $\mathcal{E}_\delta$, for $n_a(t) \ge N_a$:
\[
|\hat\mu_a(t) - \mu_a|
\le
\sqrt{\frac{\log(2KT/\delta)}{2N_a}}
\le \frac{\Delta_a}{8}.
\]
By Lemma~\ref{lem:shrinkage}:
\[
|m_a(t) - \hat\mu_a(t)| \le \frac{1}{N_a + 2} \le \frac{\Delta_a}{8}.
\]
The triangle inequality gives $|m_a(t) - \mu_a| \le \Delta_a/4$.
The same bound holds for $a^\star$ when $n_{a^\star}(t) \ge N_a$.
Therefore
\[
m_{a^\star}(t) - m_a(t)
\ge (\mu_{a^\star} - \Delta_a/4) - (\mu_a + \Delta_a/4)
= \Delta_a/2.
\]
\end{proof}

\begin{proof}[Proof of Lemma~\ref{lem:gate_main}]
After separation, $v_t \ge m_{a^\star}(t)$ and $m_{a^\star}(t) - m_a(t) \ge \Delta_a/2$, so $v_t - m_a(t) \ge \Delta_a/2$.
By the second-moment bound (Equation~\ref{eq:ei_second_moment}) applied to the Beta posterior:
\[
\EI_t(a)
\le \frac{\Var(\theta_a \mid \mathcal{H}_t)}{v_t - m_a(t)}
\le \frac{1}{2(n_a(t)+3)\,\Delta_a}.
\]
The gate requires $\EI_t(a) \cdot \surp_t(a) \ge \lambda$ with $\surp_t(a) \le L$, so necessarily $\EI_t(a) \ge \lambda/L$.
This fails once $n_a(t) > L/(2\lambda\Delta_a) - 3$.
Each override pull increments $n_a$, so the post-separation gate budget is at most $N_a^{\mathrm{gate}} = \lceil L/(2\lambda\Delta_a)\rceil$.
\end{proof}

\section{Proofs for Section~\ref{sec:small_linear}: Linear Gaussian Bandits}
\label{app:proofs_linear}

\begin{assumption}[Linear Gaussian bandit with delight gate]
\label{assump:lg}
Rewards are linear: $f(a) = x_a^\top \theta^\star$ with known features $\|x_a\|_2 \le 1$ and unknown $\theta^\star \in \mathbb{R}^d$.
Observation noise is $\mathcal{N}(0, \sigma^2)$, independent across rounds.
The prior is $\theta \sim \mathcal{N}(0, \eta^{-1}I)$ for a regularization parameter $\eta > 0$.
The host is greedy with respect to posterior means.
The gate uses price $\lambda > 0$ and cap $L > 0$.
\end{assumption}

Let $s_t^2(a) = x_a^\top \Sigma_t x_a$ denote the posterior variance of action $a$.
Let $\gamma_T := d\log(1 + T/(\eta\sigma^2 d))$ be the standard linear information-gain term.
The key point is that any selected gated override must occur on a round with substantial posterior uncertainty.

\begin{proposition}[Linear selected-override budget]
\label{prop:linear_gate_budget}
Under Assumption~\ref{assump:lg}, the number of rounds on which the override fires and selects an action from the gated set is at most
\[
N^{\mathrm{gate}}_{\mathrm{lin}}
\;\le\;
\frac{2\gamma_T}{\min\{1,\;\lambda^2/(L^2\sigma^2)\}}.
\]
In particular, the number of selected gated override rounds is independent of the number of arms $K$.
\end{proposition}

This result controls only the selected exploratory overrides, not the regret of the greedy host.
It establishes the main structural point: even in a structured model, the delight gate spends a finite uncertainty budget.
That budget scales with model complexity rather than with the number of available actions.

\noindent\textit{Proof sketch.}
Gate-open selection implies $\EI_t(a)\ge\lambda/L$.
Gaussian EI is at most the posterior standard deviation $s_t(a)$.
Thus every selected gated override consumes at least $\min\{1,\lambda^2/(L^2\sigma^2)\}$ units of elliptical potential.

We prove Proposition~\ref{prop:linear_gate_budget}.
Let $\hat\theta_t$ and $\Sigma_t$ denote the Gaussian posterior mean and covariance.
Define $\mu_t(a) = x_a^\top\hat\theta_t$ and $s_t^2(a) = x_a^\top\Sigma_t x_a$.

\begin{lemma}[Gaussian EI upper bound]
\label{lem:ei_linear}
For a greedy host with baseline $v_t = \max_a \mu_t(a)$,
\[
\EI_t(a) \le s_t(a)
\qquad\text{for all } a, t.
\]
\end{lemma}

\begin{proof}
Let $X_a \sim \mathcal{N}(\mu_t(a), s_t^2(a))$.
Since $v_t \ge \mu_t(a)$, the standardized gap $z_t(a) := (\mu_t(a) - v_t)/s_t(a)$ satisfies $z_t(a) \le 0$.
The Gaussian expected-improvement formula gives
\[
\EI_t(a) = s_t(a)\bigl[\phi(z_t(a)) + z_t(a)\,\Phi(z_t(a))\bigr].
\]
Since $\phi(z) + z\Phi(z)$ is increasing in $z$ and equals $\phi(0) = 1/\sqrt{2\pi} < 1$ at $z = 0$, the bracketed term is at most $1$ for $z \le 0$, so $\EI_t(a) \le s_t(a)$.
\end{proof}

\begin{lemma}[Elliptical potential, \citealt{abbasi2011improved}]
\label{lem:elliptical}
Under Assumption~\ref{assump:lg},
\[
\sum_{t=1}^T \min\bigl\{1,\; s_t^2(A_t)/\sigma^2\bigr\} \le 2\gamma_T.
\]
\end{lemma}

\begin{proof}
This is the standard elliptical potential bound for linear Gaussian posteriors.
See \citet{abbasi2011improved}, Lemma~11.
\end{proof}

\begin{proof}[Proof of Proposition~\ref{prop:linear_gate_budget}]
Let $\Tc_{\mathrm{sel}} := \{t \le T : \text{the override fires and } A_t \sim q_t^\lambda\}$.
If $t \in \Tc_{\mathrm{sel}}$, the selected action $A_t$ lies in the gated set, so
\[
\EI_t(A_t)\,\surp_t(A_t) \ge \lambda.
\]
Since $\surp_t(A_t) \le L$, this gives $\EI_t(A_t) \ge \lambda/L$.
By Lemma~\ref{lem:ei_linear}, $\EI_t(A_t) \le s_t(A_t)$, so $s_t^2(A_t) \ge \lambda^2/L^2$ for every $t \in \Tc_{\mathrm{sel}}$.
Each such round contributes at least $\min\{1, \lambda^2/(L^2\sigma^2)\}$ to the elliptical potential sum.
Therefore
\[
|\Tc_{\mathrm{sel}}| \cdot \min\bigl\{1,\; \lambda^2/(L^2\sigma^2)\bigr\}
\;\le\;
\sum_{t=1}^T \min\bigl\{1,\; s_t^2(A_t)/\sigma^2\bigr\}
\;\le\;
2\gamma_T,
\]
where the last inequality is Lemma~\ref{lem:elliptical}.
Rearranging gives $|\Tc_{\mathrm{sel}}| \le 2\gamma_T / \min\{1, \lambda^2/(L^2\sigma^2)\}$.
\end{proof}

\section{Proofs for Section~\ref{sec:analysis}: Large World}
\label{app:proofs_large}

Throughout this section, prospective delight is $\tilde{\delight}_t(a) = \EI_t(a)\,\surp_t(a)$.
The gated set is $\Gc_t = \{a : \tilde{\delight}_t(a) \ge \lambda\}$.
The host baseline is the current posterior best $v_t = \max_a m_a(t)$, as in Algorithm~\ref{alg:explore}.

\begin{proof}[Proof of Proposition~\ref{prop:budget}]
The proof separates tried and untried arms.

\paragraph{Tried arms after separation.}
If a tried arm $a$ has separated from the optimal arm, then Lemma~\ref{lem:gate_count} bounds its remaining selected-pull budget by
\[
  N_a^{\mathrm{gate}}
  =
  \left\lceil \frac{L}{2\lambda\Delta_a}\right\rceil .
\]
This bound depends only on $\lambda$, $L$, and $\Delta_a$, not on $K$.

\paragraph{Untried arms.}
By Lemma~\ref{lem:untried}, every untried arm satisfies $\tilde{\delight}_t(a) \le L(1-v_t)^2/2$.
If $v_t > v_{\mathrm{off}}$, then no untried arm lies in $\Gc_t$.
This is a pointwise statement: the gate criterion for any single untried arm depends only on the current $v_t$, not on how many other untried arms exist.
\end{proof}

\begin{proof}[Proof of Theorem~\ref{thm:thermostat}]
For UCB with $+\infty$ index for unpulled arms: any unpulled arm is strictly preferred to any pulled arm.
When $K \ge T$, the algorithm selects a fresh arm at every round.
For DE: by Proposition~\ref{prop:budget}, whenever $v_t > v_{\mathrm{off}}$, no untried arm lies in $\Gc_t$.
Hence every override round at such a time is restricted to tried actions.
\end{proof}

\section{Experimental Details}
\label{app:exp_details}

\paragraph{Bernoulli bandits.}
Arm means are drawn i.i.d.\ from $\mathrm{Uniform}(0,1)$ at the start of each episode.
Regret is cumulative: $\sum_{t=1}^T (f^\star - f(A_t))$.
Thompson Sampling uses exact Beta-Bernoulli conjugate updates.
All curves average over 100 independent seeds.
Shaded regions show $\pm 1$ standard error.
In the arms-scaling experiments we sweep $K$ across the values shown on the horizontal axis at fixed horizon $T = 1000$.
In the faceted comparison we use $K \in \{10, 100, 1000\}$.

\begin{figure}[ht!]
\centering
\begin{subfigure}{0.48\textwidth}
  \centering
  \includegraphics[width=\linewidth]{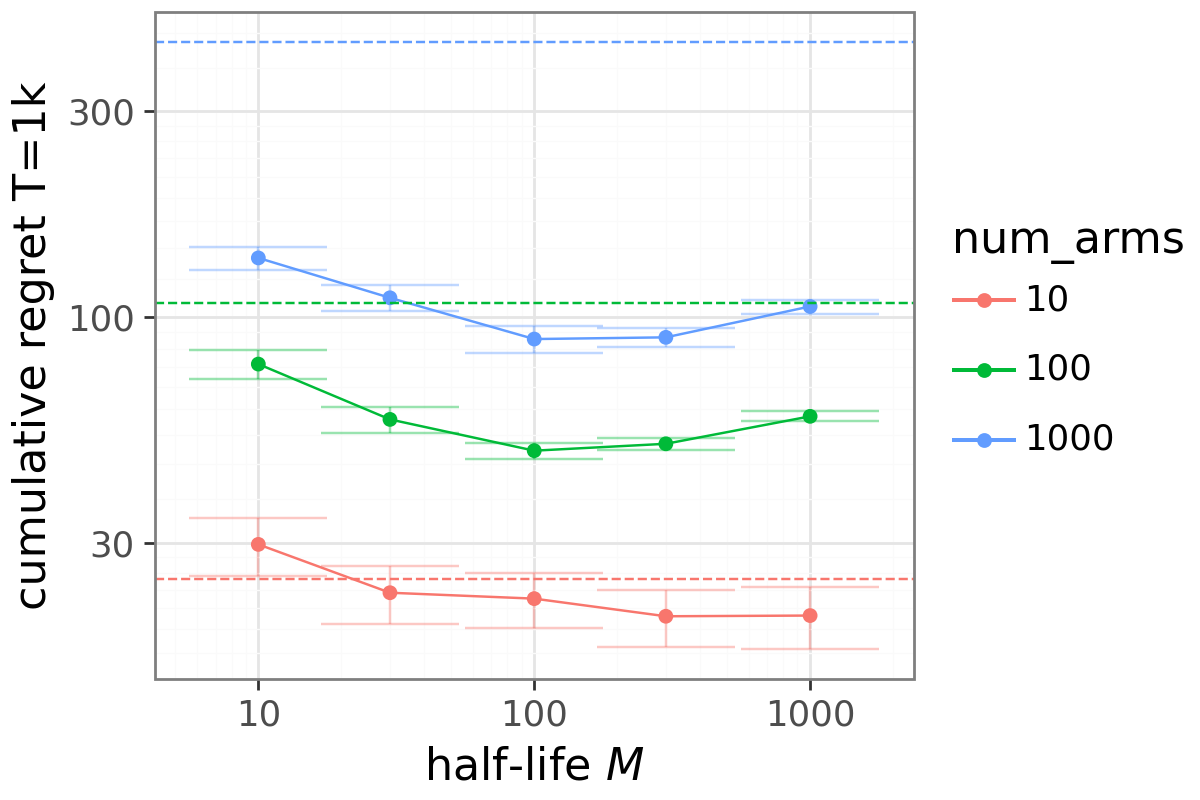}
  \caption{Sensitivity to annealing half-life $M$.}
  \label{fig:tune_M}
\end{subfigure}
\hfill
\begin{subfigure}{0.48\textwidth}
  \centering
  \includegraphics[width=\linewidth]{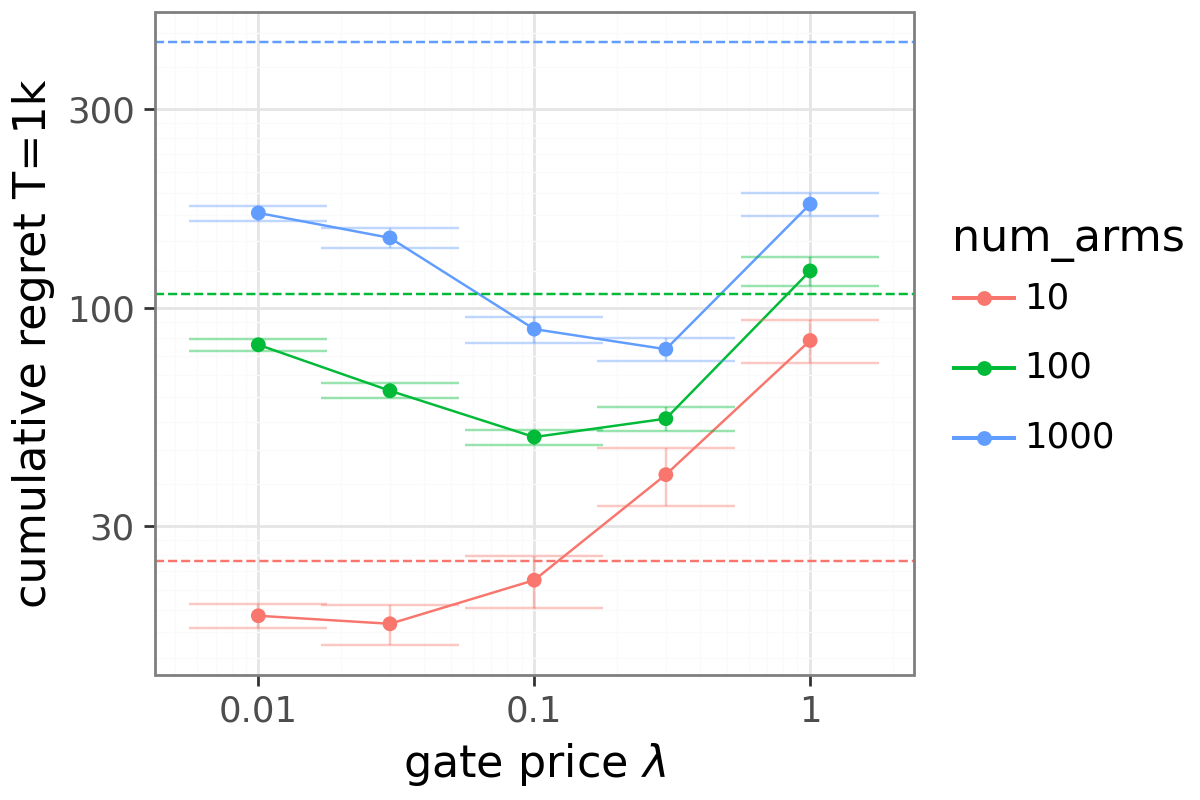}
  \caption{Sensitivity to gate price $\lambda$.}
  \label{fig:tune_lambda}
\end{subfigure}
\caption{Sensitivity to the annealing half-life $M$ and gate price $\lambda$.
Dashed lines show Thompson Sampling.
DE outperforms over a broad range of both hyperparameters, indicating that the method does not rely on fine tuning.}
\label{fig:bernoulli_tune}
\end{figure}

\paragraph{Linear bandits.}

Action features $\phi(a) \in \mathbb{R}^d$ are drawn i.i.d.\ from $\mathcal{N}(0, I_d/d)$ once per instance.
The reward parameter $\theta^\star \in \mathbb{R}^d$ is drawn from $\mathcal{N}(0, I_d)$.
Rewards are $f(a) = \phi(a)^\top \theta^\star + \eta$ with $\eta \sim \mathcal{N}(0, \sigma^2)$.
Thompson Sampling maintains the exact Gaussian posterior.
EI is computed analytically from the posterior predictive.
Surprisal uses the same cap $L$ as the Bernoulli setting.
Hyperparameters ($M=100$, $\lambda=0.1$) are transferred from Bernoulli without retuning.
In the dimension sweep we fix $K = 100$ and $\sigma = 1.0$, varying $D$.
In the noise sweep we fix $D = 30$ and $K = 100$, varying $\sigma$.

\paragraph{Tabular MDPs (DeepSea).}
We extend DE to the sequential setting using a mean-based simplification of Expected Improvement.
At the start of each episode $e$, we freeze a posterior-mean planning model $Q^{\mathrm{plan}}_e$ and define a Boltzmann host
\[
  \pi^{\mathrm{host}}_e(a\mid s)\propto \exp(Q^{\mathrm{plan}}_e(s,a)/\tau),
  \qquad \tau=0.01.
\]
During the episode, the posterior is updated online, producing an updated posterior-mean evaluator $Q^{\mathrm{post}}_{e,t}$.
We compute
\[
  \EI_{e,t}(s,a)
  =
  \left(Q^{\mathrm{post}}_{e,t}(s,a)-V^{\mathrm{plan}}_e(s)\right)_+,
  \qquad
  V^{\mathrm{plan}}_e(s)=\sum_b\pi^{\mathrm{host}}_e(b\mid s)Q^{\mathrm{plan}}_e(s,b).
\]
Surprisal is computed under the frozen host policy.

We use the same gate price $\lambda=0.1$ and surprisal cap $L=10$ as in the Bernoulli setting, transferred without retuning.
We compare against PSRL and annealed $\varepsilon$-greedy.
For the scaling plots we sweep size $H$ from 2 to 50 with 30 seeds per condition and run for $1{,}000$ episodes.

\begin{figure}[ht!]
\centering
\begin{subfigure}{0.48\textwidth}
  \centering
  \includegraphics[width=\linewidth]{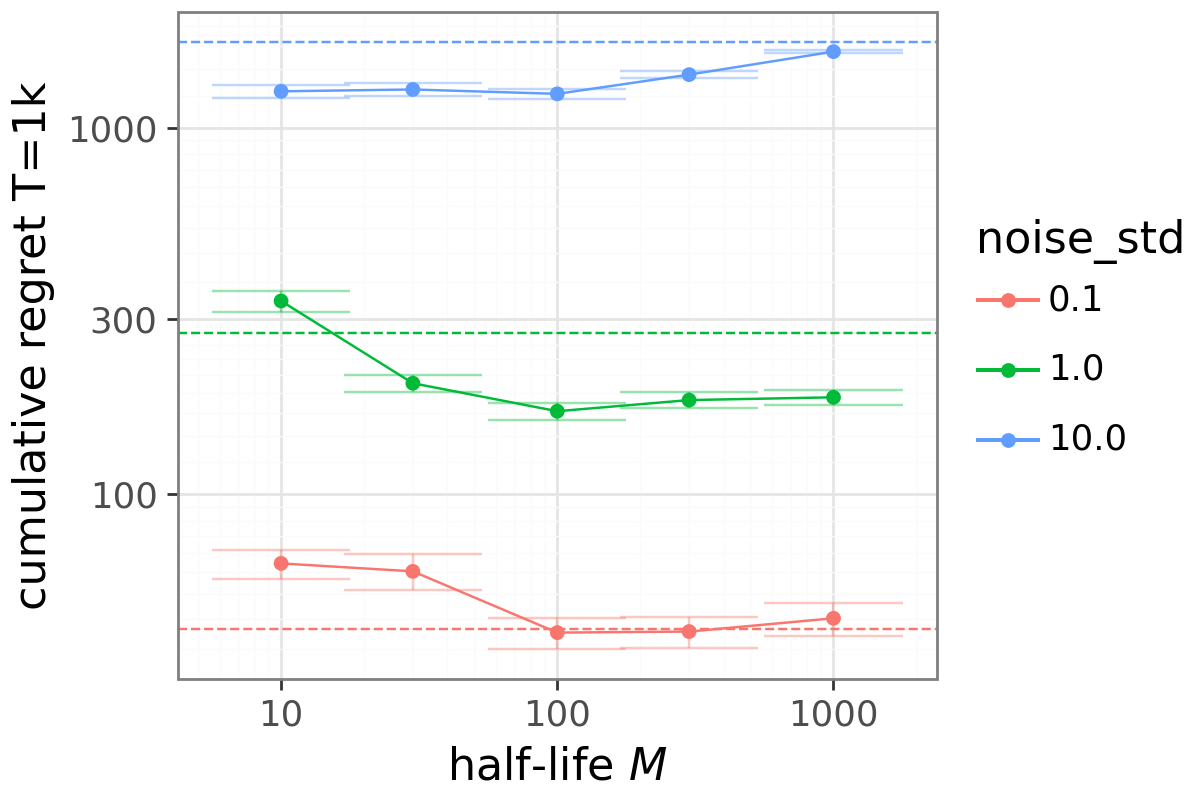}
  \caption{Sensitivity to annealing half-life $M$.}
  \label{fig:linear_tune_M}
\end{subfigure}
\hfill
\begin{subfigure}{0.48\textwidth}
  \centering
  \includegraphics[width=\linewidth]{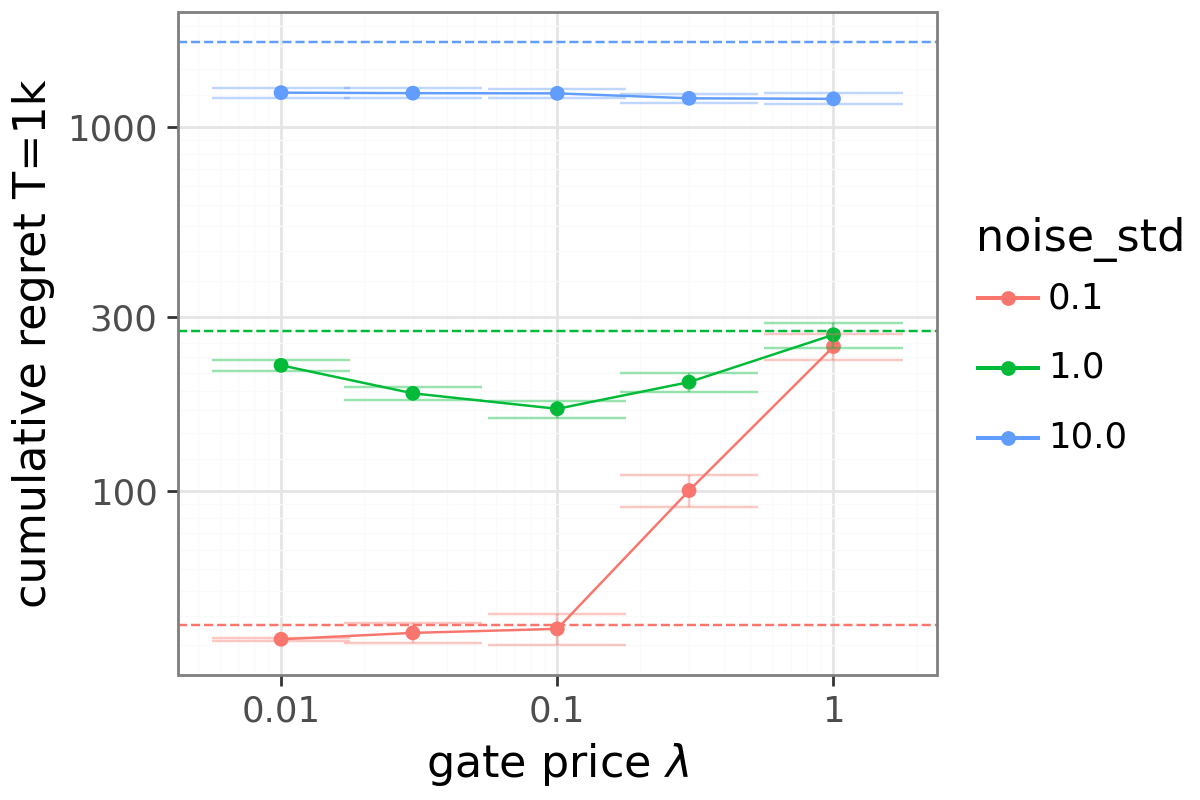}
  \caption{Sensitivity to gate price $\lambda$.}
  \label{fig:linear_tune_lambda}
\end{subfigure}
\caption{Linear bandit sensitivity to the annealing half-life $M$ and gate price $\lambda$.
Dashed lines show Thompson Sampling at corresponding noise levels.
DE outperforms over a broad range of both hyperparameters, confirming that the Bernoulli basin transfers to the linear setting.}
\label{fig:linear_tune}
\end{figure}

\begin{figure}[ht!]
\centering
\includegraphics[width=0.95\textwidth]{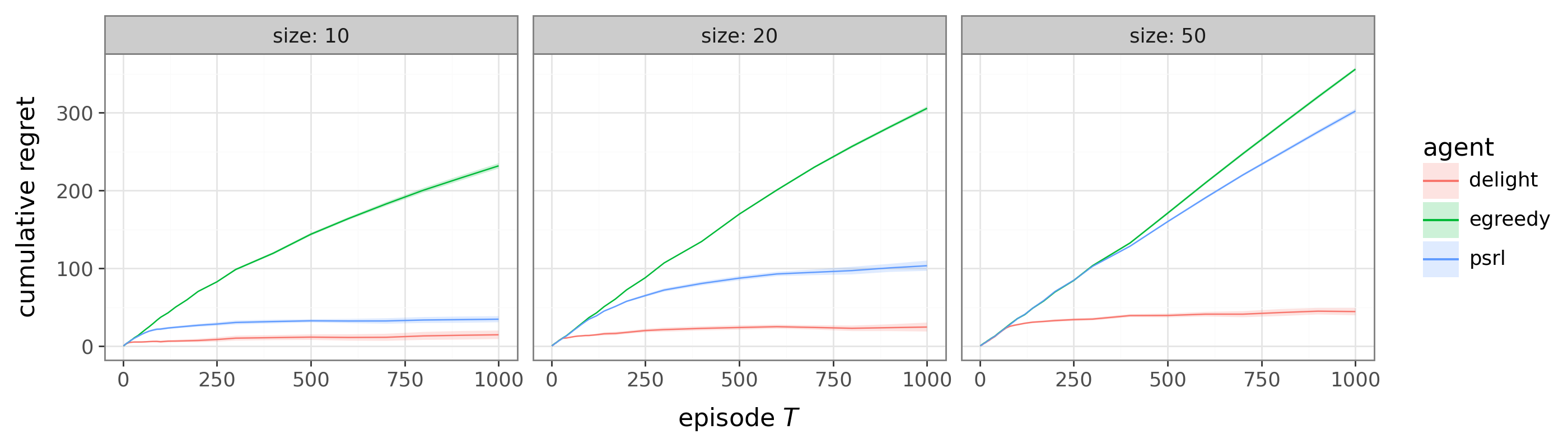}
\caption{Learning curves for DeepSea with $H \in \{5, 10, 20\}$.
DE (Delight) tracks and even outperforms PSRL, while $\varepsilon$-greedy incurs much higher regret.}
\label{fig:deepsea_curves}
\end{figure}

\section{Expected Improvement and One-Step Value of Information}
\label{app:kg_ei}

In noisy bandits, pulling an arm does not reveal its latent mean.
The following lemma shows that, for independent-arm posteriors, DE's expected-improvement term upper bounds the one-step Bayesian value of information.

\begin{lemma}[Knowledge gradient is bounded by expected improvement]
\label{lem:kg_le_ei}
Assume independent-arm posteriors, so observing action $a$ only updates the posterior of $f(a)$.
Let
\[
  \mathrm{KG}_t(a)
  :=
  \E\left[
    \bigl(\max_b m_b(t+1)-v_t\bigr)^+
    \mid \Hc_t,A_t=a
  \right]
\]
be the one-step knowledge gradient.
Then $\mathrm{KG}_t(a)\le \EI_t(a)$.
\end{lemma}

\begin{proof}
Let $Y$ be the next observation from action $a$, and write $m_a^+(Y)=\E[f(a)\mid \Hc_t,Y]$.
Since only arm $a$ is updated,
\[
  \bigl(\max_b m_b(t+1)-v_t\bigr)^+
  =
  \bigl(m_a^+(Y)-v_t\bigr)^+.
\]
The map $x\mapsto (x-v_t)^+$ is convex, so Jensen's inequality gives
\[
  \bigl(m_a^+(Y)-v_t\bigr)^+
  =
  \left(\E[f(a)\mid \Hc_t,Y]-v_t\right)^+
  \le
  \E[(f(a)-v_t)^+\mid \Hc_t,Y].
\]
Taking expectation over $Y$ and using the tower property yields
$\mathrm{KG}_t(a) \le \E[(f(a)-v_t)^+\mid \Hc_t] = \EI_t(a)$.
\end{proof}

The independence assumption is essential.
With correlated posteriors or linear bandits, observing action $a$ can update beliefs about other actions, and the action-level inequality need not hold.

\section{Proof of Theorem~\ref{thm:tail_rate}: Prior-Tail Scaling}
\label{app:tail_scaling}

Throughout, $p(y)=\Pr(X\ge 1-y)$ is the upper-tail probability, $R_T$ is the cumulative regret, and $\lambda_T$ is the corresponding DE price from~\eqref{eq:tail_price}.

\begin{proof}[Proof of Theorem~\ref{thm:tail_rate}]
For the upper bound, let $\pi_y$ sample fresh arms until observing $X\ge 1-y$, then exploit that arm.
The number of fresh samples before success is geometric with mean $1/p(y)$.
Each exploratory sample incurs regret at most one.
After success, every subsequent round incurs regret at most $y$.
Therefore $R_T(\pi_y)\le 1/p(y)+Ty$.
Using $p(y)\ge c_-y^\alpha$ gives
\[
  R_T(\pi_y)
  \le
  \frac{1}{c_-y^\alpha}+Ty.
\]
Choosing $y_T\asymp T^{-1/(\alpha+1)}$ balances the two terms and yields $R_T=O(T^{\alpha/(\alpha+1)})$.

For the lower bound, fix any policy and any $y>0$.
Let $N_T$ be the number of fresh arms sampled by time $T$, and let $\bar{r}:=1-\E[X]>0$, which is positive under the assumed nontrivial upper-tail model.
Each fresh sample has expected regret $\bar{r}$, so $R_T\ge \bar{r}\,\E[N_T]$.
If $\E[N_T]\ge 1/(4p(y))$, then $R_T\ge \bar{r}/(4p(y))$.
Otherwise $\E[N_T]<1/(4p(y))$.
By a union bound,
\[
  \Pr(\text{some sampled arm has }X\ge 1-y)
  \le p(y)\E[N_T] < 1/4.
\]
With probability at least $3/4$, no sampled arm has value at least $1-y$.
In the revealed-value discovery model, every action played by the policy is either a fresh inspection or a previously inspected arm.
On this event, all such arms have value below $1-y$, so every round incurs regret at least $y$.
Thus $R_T\ge (3/4)Ty$.
Combining the two cases,
\[
  R_T
  \ge
  \min\left\{
    \frac{\bar{r}}{4p(y)},\;
    \frac{3}{4}Ty
  \right\}.
\]
Using $p(y)\le c_+y^\alpha$ and choosing $y\asymp T^{-1/(\alpha+1)}$ gives $R_T=\Omega(T^{\alpha/(\alpha+1)})$.

For the DE price:
\[
  \lambda_T
  =
  L\int_0^{y_T}p(s)\,ds
  =
  \Theta(L\,y_T^{\alpha+1})
  =
  \Theta(L/T).
\]
With constant override rate $\varepsilon$, the expected calendar time to obtain a successful fresh inspection is $1/(\varepsilon p(y))$, so the upper bound becomes
\[
  R_T(y)\le \frac{1}{\varepsilon p(y)}+Ty.
\]
Optimizing gives
\[
  y_T\asymp(\varepsilon T)^{-1/(\alpha+1)},
  \qquad
  R_T=O\!\left(\varepsilon^{-1/(\alpha+1)}T^{\alpha/(\alpha+1)}\right),
\]
and therefore
\[
  \lambda_T
  =
  L\int_0^{y_T}p(s)\,ds
  =
  \Theta(L/(\varepsilon T)).
\]

\end{proof}

\section{Fixed-History Comparison with $\varepsilon$-Greedy}
\label{app:fixed_history}

DE and $\varepsilon$-greedy share the same host, schedule, and override probability.
They differ only in the override distribution.

\begin{theorem}[Fixed-history comparison]
\label{thm:coupling}
Fix any instance and history $\Hc_t$.
On an override round:
\begin{enumerate}
\item If $\Gc_t = \emptyset$, DE plays the host while $\varepsilon$-greedy draws uniformly from all $K$ actions.
\item If $\Gc_t \neq \emptyset$, DE draws only from actions satisfying $\EI_t(a) \ge \lambda/L$, while $\varepsilon$-greedy draws uniformly from all $K$ actions.
\end{enumerate}
\end{theorem}

\begin{proof}
Both facts follow from the algorithm definition.
In case~1, the gated set is empty, so $q_t^\lambda = \pi_t^{\mathrm{host}}$ by construction.
In case~2, every $a \in \Gc_t$ satisfies
$\tilde{\delight}_t(a)=\EI_t(a)\surp_t(a)\ge\lambda$
and $\surp_t(a)\le L$, so $\EI_t(a)\ge\lambda/L$.

\end{proof}

\begin{proposition}[Wasteful late exploration]
\label{prop:waste}
On any override round with $\Gc_t = \emptyset$, DE plays the host and incurs zero incremental regret relative to the host policy.
In contrast, $\varepsilon$-greedy's uniform override deviates from the host, paying expected incremental cost $\mu_{a_t^{\mathrm{host}}} - (1/K)\sum_a \mu_a$.
\end{proposition}

\begin{proof}[Proof of Proposition~\ref{prop:waste}]
When $\Gc_t=\emptyset$, Theorem~\ref{thm:coupling} case~1 applies: DE draws exactly the host distribution, so its incremental regret relative to the host is zero.
If the host is deterministic and selects $a_t^{\mathrm{host}}$, then $\varepsilon$-greedy replaces that action by a uniform draw over $\{1,\dots,K\}$.
Its conditional expected reward is $(1/K)\sum_{a=1}^K \mu_a$, so its incremental cost relative to the host is
\[
  \mu_{a_t^{\mathrm{host}}} - \frac{1}{K}\sum_a \mu_a .
\]
\end{proof}

\begin{proposition}[Targeted versus uniform exploration]
\label{prop:targeted}
On any override round with $\Gc_t \neq \emptyset$,
$\E[\EI_t(A_t^{\mathrm{DE}}) \mid \Hc_t] \ge \lambda/L$.
In contrast, $\varepsilon$-greedy's expected improvement is $(1/K)\sum_a \EI_t(a)$, which has no comparable lower bound.
\end{proposition}

\begin{proof}
Every $a\in\Gc_t$ satisfies
$\tilde{\delight}_t(a)=\EI_t(a)\surp_t(a)\ge\lambda$.
Since $\surp_t(a)\le L$, every such action has $\EI_t(a)\ge\lambda/L$.
DE samples only from $\Gc_t$, so its conditional expected improvement is at least $\lambda/L$.
By contrast, $\varepsilon$-greedy samples uniformly from all actions, giving expected improvement $(1/K)\sum_a\EI_t(a)$.
\end{proof}

\section{Necessity Examples}
\label{app:necessity}

The gate is not cosmetic.
A pure greedy host can lock in prematurely, while a persistent unpriced override can bleed regret forever.
This section collects simple Bernoulli examples supporting these claims.

\begin{proposition}[A gate-open override can improve on greedy (illustrative)]
\label{prop:improvement}
Under $\mathrm{Beta}(1,1)$ priors, suppose the gate opens at round $t_0$, the gated set consists of untried arms, and $v_{t_0} < 2/3$.
A single successful override pull then causes the host to switch to the discovered arm.
Consequently, on instances where this switch yields a positive mean advantage over the current host arm for sufficiently many subsequent rounds, DE incurs lower cumulative regret than pure greedy.
\end{proposition}

\begin{proof}[Proof of Proposition~\ref{prop:improvement}]
The argument is illustrative rather than a sharp lower bound.

\paragraph{Discovery opportunity.}
Under $\mathrm{Beta}(1,1)$ priors, the prior-predictive mean of an untried arm is $1/2$.
A single success updates the posterior to $\mathrm{Beta}(2,1)$ with mean $2/3$.
When $v_{t_0} < 2/3$, this update triggers a host switch.

\paragraph{One-pull switch probability.}
When $\Gc_{t_0}$ contains untried arms, all share the same delight score.
DE selects each with probability $\varepsilon/|\Gc_{t_0}|$.
The probability of a success is $\E[\mu_a \mid \Hc_{t_0}] = 1/2$ under the $\mathrm{Beta}(1,1)$ prior.

\paragraph{Benefit exceeds cost (illustrative).}
Under the $\mathrm{Beta}(1,1)$ prior, an override on an untried arm has probability $1/2$ of success.
A success promotes that arm to posterior mean $2/3$, which exceeds the current host baseline when $v_{t_0} < 2/3$ and therefore triggers a host switch.
If the discovered arm then remains better than the current host arm for sufficiently many subsequent rounds, the one-time exploration cost of order $\varepsilon$ is outweighed by the persistent gain from following the improved host.
Thus, on such instances and for large enough remaining horizon, DE can strictly improve on pure greedy.
\end{proof}

\begin{proposition}[Unpriced overrides cause linear regret]
\label{prop:nogate}
If the override fires at constant rate $\varepsilon>0$ and places nonvanishing probability on suboptimal actions after the host has identified the best arm, then $\mathrm{Regret}(T) = \Omega(\varepsilon T)$.
A blind uniform override is the simplest example.
An unpriced gate has the same pathology whenever its fallback or tie-breaking rule continues to place nonvanishing mass on suboptimal actions.
\end{proposition}

\begin{proof}[Proof of Proposition~\ref{prop:nogate}]
Suppose that after some finite time the override assigns at least $c>0$ total probability to suboptimal actions on every override round.
Let $\Delta_{\min}>0$ be the smallest positive gap among suboptimal actions that receive this mass.
Each such override round incurs expected regret at least $c\,\Delta_{\min}$.
With constant override rate $\varepsilon$, cumulative regret is at least
\[
  \varepsilon c\Delta_{\min}T - O(1)
  =
  \Omega(\varepsilon T).
\]
A blind uniform override satisfies the condition whenever at least one suboptimal action exists.
\end{proof}

\begin{proposition}[Warm host]
\label{prop:warm}
If the host assigns Boltzmann probabilities over scores in $[0,1]$ at a fixed temperature $\tau > 0$, there exist Bernoulli instances for which $\mathrm{Regret}(T) = \Omega(T)$.
\end{proposition}

\begin{proof}[Proof of Proposition~\ref{prop:warm}]
Assume the host assigns Boltzmann probabilities over scores in $[0,1]$ at temperature $\tau > 0$.
Then every arm receives selection probability at least $e^{-1/\tau}/K$, regardless of history.
On any instance with $\Delta_{\min} > 0$, the expected per-round regret from the host alone is at least $\Delta_{\min} \cdot e^{-1/\tau}/K$.
Summing over $T$ rounds gives $\mathrm{Regret}(T) \ge \Delta_{\min} T e^{-1/\tau}/K = \Omega(T)$ for fixed $\tau$ and $K$.
\end{proof}

\begin{center}
\renewcommand{\arraystretch}{1.35}
\small
\begin{tabular}{lcc}
\toprule
\textbf{Algorithm} & $T \gg K$ & $K \gg T$ \\
\midrule
UCB & $\tilde{O}(\sqrt{KT})$ & broad search (no shutoff) \\
Thompson Sampling & $\tilde{O}(\sqrt{KT})$ & posterior sampling remains broad \\
$\varepsilon$-greedy (constant $\varepsilon$) & $\Omega(\varepsilon T)$ & wastes budget; never stops \\
Greedy & $\Omega(T)$ & locks in; no recourse \\
\textbf{DE (ours)} & resolved arms shut off & \textbf{targeted; can improve on greedy} \\
\bottomrule
\end{tabular}
\end{center}

\newpage
\section{Reference Implementation}
\label{app:code}

The following self-contained Python function implements DE action selection for Bernoulli bandits with exact Beta posteriors (JAX).
The remaining experiment infrastructure (baselines, posterior updates, regret computation) uses standard components described in Appendix~\ref{app:exp_details}.

{\small
\begin{verbatim}
def select_delight(means, alphas, betas, key, t, config):
  """Delight-gated exploration with analytic Beta EI."""
  eps = config.epsilon_half_life / (config.epsilon_half_life + t)
  k1, k2 = jax.random.split(key)
  greedy = jnp.argmax(means)
  v_t = means[greedy]

  # Analytic EI for Beta posterior.
  a1 = alphas + 1.0
  ei = means * (1 - betainc(a1, betas, v_t)) \
     - v_t * (1 - betainc(alphas, betas, v_t))
  ei = jnp.maximum(0.0, ei)

  # Surprisal: L for non-greedy, 0 for greedy.
  host = jax.nn.one_hot(greedy, config.num_arms)
  surprisal = config.surprisal_cap * (1.0 - host)

  # Gate: delight >= lambda.
  delight = ei * surprisal
  in_gate = delight >= config.gate_price
  gated = delight * in_gate

  # Override or host.
  probs = gated / (gated.sum() + 1e-10)
  probs = jnp.where(jnp.any(in_gate), probs, host)
  use_override = jax.random.bernoulli(k1, eps)
  override = jax.random.categorical(k2, jnp.log(probs + 1e-10))
  return jnp.where(use_override, override, greedy)
\end{verbatim}
}


\end{document}

%% file: macros.tex
\usepackage{amsmath}
\usepackage{amssymb}
\usepackage{algorithm}
\usepackage[noend]{algpseudocode}
\usepackage{mathrsfs}
\usepackage{dsfont}
\usepackage{array}
\usepackage{bbm}
\usepackage[mathscr]{eucal}
\usepackage{psfrag}
\usepackage{here}
\usepackage{wasysym}
\usepackage[dvipsnames, table]{xcolor}  
\usepackage{caption}                     
\usepackage{subcaption}
\usepackage{amsthm}
\usepackage{todonotes}
\usepackage{tabulary}
\usepackage{outlines}
\usepackage{thmtools, thm-restate}

\newtheorem{theorem} {Theorem}

\newtheorem{proposition} {Proposition}
\newtheorem{lemma} {Lemma}

\newtheorem{assumption} {Assumption}

\DeclareMathOperator*{\argmax}{arg\,max}

\newcommand{\Gc}{\mathcal{G}}

\newcommand{\Ac}{\mathcal{A}}

\newcommand{\Tc}{\mathcal{T}}

\newcommand{\Hc}{\mathcal{H}}

\newcommand{\E}{\mathbb{E}}
\newcommand{\one}{\mathbf{1}}

\newcommand{\Var}{\mathrm{Var}}

\newcommand{\surp}{\ell} 
\newcommand{\delight}{\chi} 

\newcommand{\EI}{\mathrm{EI}} 

\definecolor{ian_highlight}{RGB}{100, 2, 2}

\errorcontextlines\maxdimen

\makeatletter
    \newcommand*{\algrule}[1][\algorithmicindent]{\makebox[#1][l]{\hspace*{.5em}\thealgruleextra\vrule height \thealgruleheight depth \thealgruledepth}}%
\newcommand*{\thealgruleextra}{}
\newcommand*{\thealgruleheight}{.75\baselineskip}
\newcommand*{\thealgruledepth}{.25\baselineskip}

\newcount\ALG@printindent@tempcnta
\def\ALG@printindent{%
    \ifnum \theALG@nested>0
        \ifx\ALG@text\ALG@x@notext
        \else
            \unskip
            \addvspace{-1pt}
            \ALG@printindent@tempcnta=1
            \loop
                \algrule[\csname ALG@ind@\the\ALG@printindent@tempcnta\endcsname]%
                \advance \ALG@printindent@tempcnta 1
            \ifnum \ALG@printindent@tempcnta<\numexpr\theALG@nested+1\relax
            \repeat
        \fi
    \fi
    }%
\usepackage{etoolbox}
\makeatother

\newbox\statebox
\newcommand{\myState}[1]{%
    \setbox\statebox=\vbox{#1}%
    \edef\thealgruleheight{\dimexpr \the\ht\statebox+1pt\relax}%
    \edef\thealgruledepth{\dimexpr \the\dp\statebox+1pt\relax}%
    \ifdim\thealgruleheight<.75\baselineskip
        \def\thealgruleheight{\dimexpr .75\baselineskip+1pt\relax}%
    \fi
    \ifdim\thealgruledepth<.25\baselineskip
        \def\thealgruledepth{\dimexpr .25\baselineskip+1pt\relax}%
    \fi
    \State #1%
    \def\thealgruleheight{\dimexpr .75\baselineskip+1pt\relax}%
    \def\thealgruledepth{\dimexpr .25\baselineskip+1pt\relax}%
}